\documentclass{article}

\usepackage{arxiv}

\usepackage[utf8]{inputenc} 
\usepackage[T1]{fontenc}    
\usepackage{hyperref}       
\usepackage{url}            
\usepackage{booktabs}       
\usepackage{amsfonts}       
\usepackage{nicefrac}       
\usepackage{microtype}      
\usepackage{lipsum}		
\usepackage{graphicx}
\usepackage{natbib}
\usepackage{doi}

\usepackage{amsthm, amsmath,amssymb,algorithm2e}

\usepackage{longtable}

\usepackage{booktabs}
\usepackage{siunitx} 
\usepackage{multirow}

\usepackage{mathtools}

\newtheorem{lem}{Lemma}[section]

\AfterEndEnvironment{lem}{\noindent\ignorespaces}
\AfterEndEnvironment{defi}{\noindent\ignorespaces}
\AfterEndEnvironment{prop}{\noindent\ignorespaces}

\newcommand{\argmin}{\operatornamewithlimits{argmin}}

\newcommand{\with}{\,  | \,}
\renewcommand{\vec}[1]{#1}
\newcommand{\fromto}{\longrightarrow}
\newtheorem{theorem}{Theorem}

\title{Conformal Prediction with Partially Labeled Data}

\author{
	Alireza Javanmardi\thanks{Corresponding author. \texttt{alireza.javanmardi@ifi.lmu.de}} \qquad Yusuf Sale \qquad Paul Hofman\qquad Eyke H\"ullermeier\\
	Institute of Informatics, LMU Munich, Germany\\
	Munich Center for Machine Learning (MCML), Germany\\ 
}

\renewcommand{\shorttitle}{Conformal Prediction with Partially Labeled Data}

\hypersetup{
pdftitle={A template for the arxiv style},
pdfsubject={q-bio.NC, q-bio.QM},
pdfauthor={David S.~Hippocampus, Elias D.~Striatum},
pdfkeywords={First keyword, Second keyword, More},
}

\begin{document}
\maketitle
\begin{abstract}
While the predictions produced by conformal prediction are set-valued, the data used for training and calibration is supposed to be precise. In the setting of superset learning or learning from partial labels, a variant of weakly supervised learning, it is exactly the other way around: training data is possibly imprecise (set-valued), but the model induced from this data yields precise predictions. In this paper, we combine the two settings by making conformal prediction amenable to set-valued training data. We propose a generalization of the conformal prediction procedure that can be applied to set-valued training and calibration data. We prove the validity of the proposed method and present experimental studies in which it compares favorably to natural baselines.

\end{abstract}


\section{Introduction}
\label{sec:intro}
Conformal prediction (CP), a prominent uncertainty quantification technique, has drawn increasing attention in statistics and machine learning over the past decade. With its roots in classical frequentist statistics, this framework enables the construction of reliable prediction sets without the need for any distributional assumptions \citep{vovk2005algorithmic}. A key advantage of conformal prediction lies in its validity guarantees for the constructed prediction sets, which cover the true outcomes with high probability. This makes it appealing for applications in safety-critical domains, such as risk assessment in finance \citep{gammerman2007hedging}, medical diagnosis and disease prediction \citep{papadopoulos2009reliable}, drug discovery and toxicity prediction \citep{svensson2018conformal}, among many others.

Another machine learning (ML) setting dealing with set-valued data is partial label learning (PLL), a specific type of weakly supervised learning \citep{grandvalet2002logistic,jin2002multiple,nguyen2008classification,cour2011partial}. In a sense, PLL is orthogonal to conformal prediction: While the predictions produced by CP are set-valued, the data used for training and calibration is supposed to be precise. In PLL, it is exactly the other way around: Although the training data might be imprecise (set-valued), the goal is to induce a unique model producing precise (point) predictions. This may strike as odd, as one may argue that if the training data is imprecise or ambiguous, it might be all the more important to reflect this imprecision or ambiguity in the induced model and the predictions produced by this model.  
For example, one may allow for a set of incomparable, undominated models, resulting, for instance, from the interval order induced by set-valued loss functions~\citep{couso2016machine}, or by the application of conservative, imprecise Bayesian updating rules~\citep{zaffalon2009conservative}.

As an alternative, we suggest the use of CP to capture (predictive) uncertainty in the setting of PLL. In other words, we propose to combine PLL and CP within a single framework. To this end, we propose a generalization of the CP procedure that can be applied to set-valued training and calibration data. For this approach, we establish theoretical validity guarantees. Moreover, through experimental studies, we showcase the enhanced accuracy of our method in weakly supervised learning settings compared to natural baselines. 

\section{Background}

\subsection{Partial Label Learning}
\label{sec:pll}

As already said, partial label learning (PLL) is a specific type of learning from weak supervision, in which the outcome (response) associated with a training instance is only characterized in terms of a subset of possible candidates. Thus, PLL is somehow in-between supervised and semi-supervised learning, with the latter being a special case. Motivated by practical applications in which only partial information about outcomes is available, PLL has been studied under various names, such as \emph{learning from ambiguously labeled examples} \citep{hullermeier2006ambiguous}, and under slightly different assumptions on the incomplete information being provided. Often, the only assumption made is that the set of candidates covers the actual (precise) outcome, which is also reflected by the name \emph{superset learning} \citep{liu2012conditional}.

More formally, consider the standard setting of supervised learning with a data space $\mathcal{Z} = \mathcal{X} \times \mathcal{Y}$, where $\mathcal{X}$ is the instance space and $\mathcal{Y}$ is the output space, respectively. As our focus is on the multi-class classification scenario, we refer to the output variable as ``label'' and assume $\mathcal{Y}$ to be finite (typically of small to moderate size). 
The learning task normally consists of choosing an optimal model (hypothesis) $h^*$ from a given model space (hypothesis space) $\mathcal{H}$, based on a set of training data
\begin{equation}\label{eq:pdata}
\mathcal{D} = \big\{ \, (\vec{x}_i , y_i ) \, \big\}_{i=1}^n \, \in \, (\mathcal{X}\times\mathcal{Y})^n \enspace .
\end{equation}
More specifically, optimality typically refers to prediction accuracy, i.e., a model is sought whose expected prediction loss or \emph{risk} 
\begin{equation}\label{eq:risk}
\mathcal{R}(h) \, = \, \mathbb{E}_{(\vec{x},y) \sim P} \, L \big( y , h(\vec{x}) \big) \, = \,  
\int L \big( y , h(\vec{x}) \big) \, d \, P(\vec{x}, y) 
\end{equation}
is minimal; here, $L: \, \mathcal{Y} \times \mathcal{Y} \fromto \mathbb{R}$ is a loss function, and $P$ is an (unknown) probability measure on $\mathcal{X} \times \mathcal{Y}$ modeling the underlying data generating process. 

In PLL, the learning algorithm does not have direct access to the data (\ref{eq:pdata}) because the labels $y_i \in \mathcal{Y}$ are not observed precisely. Instead, only supersets $S_i \subseteq \mathcal{Y}$ are observed so that the training data consists of (imprecise, coarse, ambiguous) observations
\begin{equation}\label{eq:idata}
\mathcal{O} = \big\{ (\vec{x}_i , S_i) \big\}_{i=1}^n \, \in \, (\mathcal{X}\times 2^\mathcal{Y})^n \enspace .
\end{equation}
There are various ways of learning from data of that kind, notably the idea of generalizing the principle of empirical risk minimization through the use of a generalized loss function. For example, \cite{hullermeier2014disambiguation} introduces the \emph{optimistic superset loss} as an extension of the loss $L$ in (\ref{eq:risk}): 
\begin{equation}\label{eq:app1}
L_O(S, \hat{y}) =  \min \big\{ L(y, \hat{y}) \with y \in S \big\}  \enspace .
\end{equation}
Learning is then accomplished by finding a model minimizing this loss (or maybe a regularized version thereof) on the training data:
\begin{equation}\label{eq:app2}
h^* \in 
\argmin_{h \in \mathcal{H}} \, \frac{1}{n} \sum_{i=1}^n 
L_O\big( S_i , h (\vec{x}_i) \big) \enspace .
\end{equation}
A key motivation of this approach is the idea of \emph{data disambiguation}, i.e., the idea of simultaneously inducing the true model and reconstructing the values of the underlying precise data. The same type of loss function has more recently been introduced under the notion of \emph{infimum loss} \citep{cabannnes2020structured}.

Obviously, depending on the underlying loss function $L$, the optimization problem (\ref{eq:app2}) may become complex, especially since (\ref{eq:app1}) could be non-convex. From a theoretical perspective, an important question concerns conditions under which successful learning (for example, in the sense of convergence toward a truly optimal model) is actually possible, despite the imprecision of the data. An analysis of this kind obviously requires assumptions about the process of ``imprecisiation'', i.e., the way in which precise outcomes are turned into imprecise observations. The first positive results, showing that successful learning is possible under specific assumptions, have been obtained by \cite{liu2014superset,cabannnes2020structured,cabannnes2021disambiguation}. 
\subsection{Conformal Prediction}
\label{sec:cp}

Suppose a training dataset (\ref{eq:pdata}) to be given, and denote by $(x_{new},y_{new}) \in \mathcal{Z}$ a new test point. 
Assuming that $x_{new}$ is observed, but $y_{new}$ is not, CP aims to construct a \emph{prediction set} of the form $\mathcal{T}(x_{new}) \subseteq \mathcal{Y}$ that is valid in the sense that $y_{new} \in \mathcal{T}(x_{new})$ with high probability. Informally speaking, the idea of CP is to test the hypothesis $y_{new} = y$ for all $y \in \mathcal{Y}$ and to exclude from the prediction set only those outcomes $y$ for which this hypothesis can be rejected at the predefined level of confidence. Hypothesis testing is done in a nonparametric way: Consider any \emph{nonconformity function} that assigns scores $\alpha(\vec{x}, y)$ to input/output tuples; the latter can be interpreted as a measure of ``strangeness'' of the pair $(\vec{x}, y)$, i.e., the higher the score, the less the data point $(\vec{x}, y)$ conforms to what one would expect to observe. Assuming \emph{exchangeability} of the data $\mathcal{D}$, CP then finds a critical value $q$ for the degree of nonconformity so that those $y$ with $\alpha(\vec{x}, y) > q$ are excluded. Theoretically, the CP procedure is able to guarantee \emph{marginal coverage}, meaning that, in an infinite sequence of predictions, the miscoverage rate (fraction of predictions $\mathcal{T}(x_{new})$ not covering $y_{new}$) does not exceed a prespecified value $\epsilon > 0$.

This guarantee holds true regardless of how the nonconformity function is defined. Yet, this function has a strong influence on the \emph{efficiency} of predictions, i.e., the (average) size of the prediction sets. A common approach, which we will also assume in the following, is to train a probabilistic predictor $\hat{f}$ so that $\hat{f}(x)$ is a prediction of the conditional probability $p( \cdot \, | \, x)$ on $\mathcal{Y}$. Nonconformity scores are then naturally defined in terms of reciprocals of class probabilities, i.e., $\alpha(x,y) = 1 - \hat{f}(x)_{y}$. 

CP has originally been developed in a transductive setting, which, however, comes with major computational challenges (e.g., one would need to retrain the predictor $\hat{f}$ after each new observation). 
Later on, inductive variants of CP have also been developed \citep{papadopoulos2002inductive, papadopoulos2002qualified}.
To construct prediction sets using inductive conformal prediction (ICP), the first step is to partition the training data $\mathcal{D}$ into two subsets, the \emph{proper training set} $\mathcal{D}_\text{train}$
and the \emph{calibration set} $\mathcal{D}_\text{calib}$:
\begin{align*}
\mathcal{D}_\text{train} & = \lbrace (x_i, y_i):i\in \mathcal{I}_1 \rbrace \\
\mathcal{D}_\text{calib} & = \lbrace (x_i, y_i): i\in \mathcal{I}_2 \rbrace
\end{align*}
Then, a multi-class classification algorithm $\mathcal{A}$ is used to fit a (probabilistic) predictor to the proper training set:
\begin{align}\label{eq:classifier:precise}
    \hat{f}(\cdot) \leftarrow \mathcal{A} (\mathcal{D}_\text{train})
\end{align}
The next step is called calibration, which involves computing the nonconformity score of each calibration data instance that determines how well it conforms to the established classifier $\hat{f}$. As already said, a natural choice for the nonconformity is one minus the predicted probability of the ground-truth class, giving rise to a score set    
\begin{align}\label{eq:nonconformitySet:precise}
    \mathcal{E} := \left\lbrace \, 1 - \hat{f}(x_j)_{y_j}: j \in \mathcal{I}_2 \, \right\rbrace \, . 
\end{align}
For any set of nonconformity scores $\mathcal{E}$, define the critical score $q(\mathcal{E}, \epsilon)$ in terms of its $\lceil (1+|\mathcal{E}|)(1-\epsilon)\rceil$ smallest value, or equivalently, its $|\mathcal{E}|^{-1} \lceil (1+|\mathcal{E}|)(1-\epsilon)\rceil$ empirical quantile. Furthermore, given an instance $x \in \mathcal{X}$, a classifier $\hat{f} : \mathcal{X} \fromto \mathbb{P}(\mathcal{Y})$, a set of nonconformity scores $\mathcal{E}$, and an error rate $\epsilon$, define the prediction set $\mathcal{T}(x, \hat{f}, \mathcal{E}, \epsilon)$ as
\begin{align}\label{eq:prediction_set:general}
    \mathcal{T}(x, \hat{f}, \mathcal{E}, \epsilon) := \left\lbrace y\in\mathcal{Y}: \hat{f}(x)_{y} \geq 1 -  q(\mathcal{E}, \epsilon) \right\rbrace  \, .
\end{align}
ICP outputs $\mathcal{T}(x_{new}, \hat{f}, \mathcal{E}, \epsilon)$ as the prediction set $\hat{Y}_{new}$ for a new test point $x_{new}$, thereby satisfying the marginal coverage property (in expectation) if samples in $\mathcal{D}_\text{calib}\cup \{(x_{new}, y_{new})\}$ are drawn exchangeably from a joint probability distribution over the data space.

\section{Conformal Prediction with Partially Labeled Data}

Coming back to the idea of combining conformal prediction with partial label learning, we are now again interested in the case of set-valued training data (\ref{eq:idata}), where each data instance $x_i$ is associated with a set of potential labels $S_i \subseteq \mathcal{Y}$.
We keep denoting the ground-truth label of instance $x_i$ by $y_i$ and assume it lies in its candidate set, i.e., $y_i \in S_i$. 

Applying ICP to such data, we again start by partitioning the data $\mathcal{O}$ into proper training $\mathcal{O}_\text{train} = \lbrace (x_i, S_i):i\in \mathcal{I}_1 \rbrace$ and calibration subsets $\mathcal{O}_\text{calib} = \lbrace (x_i, S_i):i\in \mathcal{I}_2 \rbrace$. As mentioned in Section \ref{sec:pll}, the task of learning from partially labeled data has been well-studied in the literature, making the training step of ICP feasible for such data. All we need to do next is replace the algorithm $\mathcal{A}$ in \eqref{eq:classifier:precise} with a partial label learning algorithm $\mathcal{A}_\text{PLL}$ that allows us to fit a classifier on $\mathcal{O}_\text{train}$:
\begin{align}\label{eq:classifier:imprecise}
    \hat{f}_\text{PLL} (\cdot) \leftarrow \mathcal{A}_\text{PLL} (\mathcal{O}_\text{train})
\end{align}
Like before, we assume that the induced classifier $\hat{f}_\text{PLL} (\cdot)$ predicts probability distributions over the classes. However, each calibration instance $x_j$ is now associated with (possibly) multiple labels $S_j$. This begs the general question of how to compute the nonconformity scores for set-valued data $(x_i , S_i)$. 

A relative straightforward approach is to consider each candidate label $y \in S_j$ separately, compute its nonconformity score $1 - \hat{f}_\text{PLL} (x_j)_{y}$, and then pessimistically pick the maximum one:
\begin{align}\label{eq:nonconformitySet:imprecise:max}
    \mathcal{E}_\text{max} := \left\lbrace 1 - \min_{y \in S_j}\hat{f}_\text{PLL} (x_j)_{y}: j \in \mathcal{I}_2 \right\rbrace. 
\end{align}
For a new test point $x_{new}$ and any $\epsilon \in (0,1]$, the prediction set is then given by 
\begin{equation}\label{eq:predmax}
\hat{Y}_{new} = \mathcal{T}(x_{new}, \hat{f}_\text{PLL}, \mathcal{E}_\text{max}, \epsilon) \, .
\end{equation}
Let $\mathcal{O}_\text{calib}' = \lbrace (x_i, y_i) :  i\in \mathcal{I}_2 \rbrace$ be the precise counterpart of $\mathcal{O}_\text{calib}$, i.e., the underlying precise data that the PLL learner could not observe. Moreover,  define 
$$
\mathcal{E}_1 := \left\lbrace 1 - \hat{f}_\text{PLL}(x_j)_{y_j}: j \in \mathcal{I}_2 \right\rbrace  \, . 
$$
The following theorem establishes the validity of the prediction sets made by this approach. 
\begin{theorem}\label{theorem:max}
      If the data points in $\mathcal{O}_\text{calib}' \cup (x_{new}, y_{new})$ are exchangeable, then the prediction set (\ref{eq:predmax}) with underlying score set (\ref{eq:nonconformitySet:imprecise:max}) satisfies 
\begin{align*}
    \mathbb{P}\Big(y_{new} \in \mathcal{T}(x_{new}, \hat{f}_\text{PLL}, \mathcal{E}_\text{max}, \epsilon)\Big) \geq 1-\epsilon \, .
\end{align*}
\end{theorem}
\begin{proof}
     The vanilla CP guarantees that the prediction set $\mathcal{T}(x_{new}, \hat{f}_\text{PLL}, \mathcal{E}_1, \epsilon)$ is valid. To conclude the proof, we show that $\mathcal{T}(x_{new}, \hat{f}_\text{PLL}, \mathcal{E}_1, \epsilon) \subseteq \mathcal{T}(x_{new}, \hat{f}_\text{PLL}, \mathcal{E}_\text{max}, \epsilon)$. To this end, it is enough to show that $q(\mathcal{E}_\text{max}, \epsilon)\geq q(\mathcal{E}_1, \epsilon)$, which immediately follows from
    \begin{align*}
        1 - \min_{y \in S_j}\hat{f}_\text{PLL}(x_j)_{y} \geq 1 - \hat{f}_\text{PLL}(x_j)_{y_j}  
    \end{align*}
   for all  $j \in \mathcal{I}_2$.
\end{proof}

Although the mentioned pessimistic approach preserves the validity of CP, it usually ends up in unnecessarily large prediction sets. Here we suggest another approach by incorporating the nonconformity scores of all candidate labels of each calibration instance. Indeed, consider the following set:
\begin{align}\label{eq:nonconformitySet:imprecise:all}
    \mathcal{E}_\text{all} := \left\lbrace 1 - \hat{f}_\text{PLL}(x_j)_{y}: j \in \mathcal{I}_2 \text{ and }  y \in S_j \right\rbrace
\end{align}
Compared to the previous cases, $\mathcal{E}_\text{all}$ has a bigger cardinality. The following theorem shows that $\mathcal{T}(x_{new}, \hat{f}_\text{PLL}, \mathcal{E}_\text{all}, \epsilon)$ is also valid under certain assumptions. 

\begin{theorem}
For any $\epsilon \leq  \min \bigg( \dfrac{1}{4}, \dfrac{|\mathcal{O}_\text{calib}| + |\mathcal{Y}|}{|\mathcal{Y}| \cdot (1+|\mathcal{O}_\text{calib}|)}$\bigg), if the points in $\mathcal{O}_\text{calib}' \cup \{(x_{new}, y_{new})\}$ are exchangeable and $q(\mathcal{E}_1, \epsilon) \leq 0.5$, then the prediction set $\mathcal{T}(x_{new}, \hat{f}_\text{PLL}, \mathcal{E}_\text{all}, \epsilon)$ is valid. 
\end{theorem}
\begin{proof}
Again, we prove this result by showing $\mathcal{T}(x_{new}, \hat{f}_\text{PLL}, \mathcal{E}_1, \epsilon) \subseteq \mathcal{T}(x_{new}, \hat{f}_\text{PLL}, \mathcal{E}_\text{all}, \epsilon)$.
We start with the set $\mathcal{E}_1$ and show that by adding the nonconformity scores of the other (false) candidates $y \in S_j \setminus \{ y_j \}$, its critical score can only get larger. This immediately follows from the observation that most of the nonconformity scores of false candidates are going to be added ``to the right'' of $q(\mathcal{E}_1, \epsilon)$, i.e., they exceed this value. 

First, observe that for any instance $(x_j, S_j) \in \mathcal{O}_\text{calib}$, the nonconformity score for any false label in a candidate set is at least as great as the predicted probability of the true class, i.e.,
\begin{align}\label{eq:false_nonconformity}
     1 - \hat{f}_\text{PLL}(x_j)_{y} \geq 1 - \max_{y \in S_j \setminus \lbrace y_j\rbrace} \hat{f}_\text{PLL}(x_j)_{y} \geq \hat{f}_\text{PLL}(x_j)_{y_j}, ~~ \forall y \in S_j \setminus \lbrace y_j\rbrace.
\end{align}
Let $t := \lceil (1+|\mathcal{E}_1|)(1-\epsilon)\rceil$. Take $(x_l, y_l)$ as an instance that is among the $t$ smallest elements of $\mathcal{E}_1$. By definition, we have $1 - \hat{f}_\text{PLL}(x_l)_{y_l} \leq q(\mathcal{E}_1, \epsilon)$, which implies $\hat{f}_\text{PLL}(x_l)_{y_l} \geq 1 - q(\mathcal{E}_1, \epsilon)$. Since $q(\mathcal{E}_1, \epsilon) \leq 0.5$, we have $\hat{f}_\text{PLL}(x_l)_{y_l} \geq q(\mathcal{E}_1, \epsilon)$. This, together with \eqref{eq:false_nonconformity}, tells us that the nonconformity scores of all false labels will be located to the right of $q(\mathcal{E}_1, \epsilon)$. Hence for these $t$ points, there will be at least $t$ scores added to the right of $q(\mathcal{E}_1, \epsilon)$\footnote{It is obvious that the instance $(x_l, S_l)$ will add $|S_l-1|$ scores to the right of $q(\mathcal{E}_1, \epsilon)$.}.  

Now take $(x_k, y_k)$ as an instance which is among the $|\mathcal{E}_1| - t$ greatest elements of $\mathcal{E}_1$. Again by definition, we have $\hat{f}_\text{PLL}(x_k)_{y_k} < 1- q(\mathcal{E}_1, \epsilon)$. For the sake of worst-case consideration,  we can assume that $\hat{f}_\text{PLL}(x_k)_{y_k} < q(\mathcal{E}_1, \epsilon)$ as well so that the nonconformity scores of all false labels of these $|\mathcal{E}_1| - t$  instances will be located on the left-hand side of $q(\mathcal{E}_1, \epsilon)$. Therefore, for these $|\mathcal{E}_1| - t$ points, there will be at most $(|\mathcal{E}_1| - t) \cdot (\mathcal{Y}-1)$ scores added to the left-hand side of $q(\mathcal{E}_1, \epsilon)$.  

Since $\epsilon\leq \dfrac{1}{4}$, all we need for $q(\mathcal{E}_\text{all}, \epsilon)$ to be greater than or equal to $q(\mathcal{E}_1, \epsilon)$ is that the number of scores added to the right of $q(\mathcal{E}_1, \epsilon)$ to be greater than or equal to those added to its left-hand side (See Lemma \ref{lemma} for details), which is the case since  

\begin{align*}
\epsilon &\leq  \frac{|\mathcal{O}_\text{calib}| + |\mathcal{Y}|}{|\mathcal{Y}| \cdot (1+|\mathcal{O}_\text{calib}|)} \Rightarrow \\
1-\epsilon &\geq  \frac{|\mathcal{O}_\text{calib}|\cdot (|\mathcal{Y}|-1)}{|\mathcal{Y}| \cdot (1+|\mathcal{O}_\text{calib}|)} \Rightarrow \\
t &\geq  \frac{|\mathcal{O}_\text{calib}|\cdot (|\mathcal{Y}|-1)}{|\mathcal{Y}|}  \Rightarrow \\
t & \geq (|\mathcal{O}_\text{calib}| - t) \cdot (|\mathcal{Y}|-1)
\end{align*}
\end{proof}

The condition $q(\mathcal{E}_1, \epsilon) \leq 0.5$ implies that the error rate of the induced classifier on calibration data must not exceed $\epsilon \cdot 100\%$. Later on, in Section \ref{sec:exp}, we will see that the validity of prediction sets generated by this method holds even when this condition is violated, and also, compared to the pessimistic approach $\mathcal{E}_\text{max}$, this method results in more efficient (i.e., smaller) prediction sets.

Another possibility is to consider the average of the nonconformity scores per calibration instance. Hence, the set of nonconformity scores would be  
\begin{align}\label{eq:nonconformitySet:imprecise:mean}
    \mathcal{E}_\text{mean} := \lbrace 1 - \frac{\sum_{y\in S_j}\hat{f}_\text{PLL}(x_j)_{y}}{|S_j|}: j \in \mathcal{I}_2 \rbrace.
\end{align}
The validity of the prediction set $\mathcal{T}(x_{new}, \hat{f}_\text{PLL}, \mathcal{E}_\text{mean}, \epsilon)$ is established by the following theorem.
\begin{theorem}
If the data points in $\mathcal{O}_\text{calib}' \cup (x_{new}, y_{new})$ are exchangeable and $\hat{f}_\text{PLL}(x_j)_{y_j} \geq \dfrac{1}{|S_j|}, ~ \forall j \in \mathcal{I}_2$, then the prediction set $\mathcal{T}(x_{new}, \hat{f}_\text{PLL}, \mathcal{E}_\text{mean}, \epsilon)$ is valid. 
\end{theorem}
\begin{proof}
$\mathcal{T}(x_{new}, \hat{f}_\text{PLL}, \mathcal{E}_1, \epsilon) \subseteq \mathcal{T}(x_{new}, \hat{f}_\text{PLL}, \mathcal{E}_\text{mean}, \epsilon)$ holds because
    \begin{align*}
    \hat{f}_\text{PLL}(x_j)_{y_j} \geq \dfrac{1}{|S_j|}&\geq \frac{\sum_{y\in S_j}\hat{f}_\text{PLL}(x_j)_{y}}{|S_j|},~~\forall j \in \mathcal{I}_2 \Rightarrow\\
     1 - \frac{\sum_{y\in S_j}\hat{f}_\text{PLL}(x_j)_{y}}{|S_j|} &\geq  1 - \hat{f}_\text{PLL}(x_j)_{y_j} ,~~\forall j \in \mathcal{I}_2 \Rightarrow\\
        q(\mathcal{E}_\text{mean}, \epsilon) &\geq q(\mathcal{E}_1, \epsilon).  
\end{align*}
    
\end{proof}

The requirement of this theorem is demanding, especially for the case where most candidate sets consist of only two labels. However, it should be noted that it is likely to have valid prediction sets using this approach, even if, for some calibration instances, the average nonconformity score over the candidate set falls below the nonconformity score of the precise counterpart. Furthermore, as the cardinality of the candidate sets increases, this requirement becomes less burdensome.

\section{Experiments}\label{sec:exp}
In this section, we evaluate the performance of the proposed frameworks numerically. Apart from the three approaches mentioned in previous sections, there are other approaches to form the set of nonconformity scores that are not necessarily coming with a coverage guarantee. Here, we bring two of them, which we consider in our comparisons as well: 
\begin{itemize}
    \item Taking minimum nonconformity score per calibration instance:
    \begin{align}\label{eq:nonconformitySet:imprecise:min}
    \mathcal{E}_\text{min} := \left\lbrace 1 - \max_{y\in S_j}\hat{f}_\text{PLL}(x_j)_{y}: j \in \mathcal{I}_2 \right\rbrace.
\end{align}
This optimistic approach is a natural baseline for the comparison. It can be seen as \textit{calibration with the disambiguated data}, where the induced classifier is utilized to disambiguate the calibration data, and the calibration proceeds as in the vanilla conformal prediction.
\item Taking the weighted average of minimum and maximum nonconformity score per calibration instance:
\begin{align}\label{eq:nonconformitySet:imprecise:mu}
    \mathcal{E}_\mu := \left\lbrace \mu\cdot \Big(1 - \max_{y\in S_j}\hat{f}_\text{PLL}(x_j)_{y} \Big) + (1 - \mu)\cdot \Big(1 - \min_{y\in S_j}\hat{f}_\text{PLL}(x_j)_{y} \Big): j \in \mathcal{I}_2 \right\rbrace,
\end{align}
with $\mu \in [0,1]$ being a hyperparameter. This approach lies between the pessimistic and optimistic ones. With a proper selection of $\mu$, one might be able to achieve prediction sets that are both valid and efficient.  
\end{itemize}

Our implementation code is publicly available on  \href{https://github.com/pwhofman/conformal-partial-labels}{GitHub}\footnote{\url{https://github.com/pwhofman/conformal-partial-labels}} to enable the reproducibility of the presented results. 
\subsection{Datasets}
\begin{table}[t]
\caption{Description of the benchmark and real datasets.}
\label{tab:datasets:details}
\resizebox{\columnwidth}{!}{
\begin{tabular}{@{}llcccccccc@{}}
\toprule
                        & & FashionMNIST & KMNIST & MNIST & BirdSong & Lost & MSRCv2 & Soccer Player & Yahoo!News \\ \midrule
& Num. of classes         & 10           & 10     & 10    & 13       & 16   & 23     & 171           & 219        \\ \midrule
\parbox[t]{2mm}{\multirow{4}{*}{\rotatebox[origin=c]{90}{Avg. CSS}}} & Original &         -     &    -    &   -    &  2.18        &   2.23   &  3.16      &    2.09           &  1.91  \\
& Instance-dependent contamination &         2.32     &    2.49  & 2.25    &  - &  -&  -  & - & -  \\
& Random contamination (p=0.1)  &         2.29     &  2.29      & 2.29      &   -       &  -    &    -    &          -     &    -        \\
& Random contamination (p=0.7) &       7.30       &   7.30     &    7.30   &    -      &   -   &     -   &    -           &     -       \\ 


\bottomrule
\end{tabular}
}
\end{table}


Experiments are performed using the benchmark datasets: MNIST \citep{lecun1998gradient}, Kuzushiji-MNIST \citep{clanuwat2018deep}, and Fashion-MNIST \citep{xiao2017fashionmnist}. However, these datasets are all precise and need to be synthetically contaminated. We use the following two methods to convert these datasets into partially labeled data:
\begin{itemize}
    \item Random contamination: In this method, we create a candidate set for each instance in a random manner by including each non-ground-truth label with probability $p$. In cases where no label among the non-ground-truth labels is added to the set, a random label is added to ensure all data is partially labeled.
    \item Instance-dependent contamination: Similar to \cite{ning2021instance}, we train a simple classifier $\hat{f}_s$ on each benchmark dataset which we refer to as the supermodel. For each dataset, we exploit its supermodel to compute the probability of adding each non-ground-truth label $y \in \mathcal{Y} \setminus \{y_i\}$ to the candidate set of instance $x_i$ as $p_y = \frac{\hat{f}_s(x_i)_y}{\max_{y \in \mathcal{Y} \setminus \{y_i\}}\hat{f}_s(x_i)_y}$. More details about the supermodels can be found in \ref{sec:appendix:supermodel}. 
\end{itemize}

In addition to the synthetic datasets, we adopt five commonly used real-world partial label datasets: Lost \citep{cour2011partial}, MSRCv2 \citep{liu2012conditional}, BirdSong \citep{brigs2012rank}, Soccer Player \citep{zeng2013} and Yahoo!News \citep{guillaumin2010multiple}. Table \ref{tab:datasets:details} gives an overview of the benchmark and real-world datasets, including the number of classes and the average candidate set sizes (CSS).


\subsection{Models}
We bring a well-known partial label learning algorithm, PRODEN \citep{lv2020progressive}, which progressively tries to find the ground-truth label and adjusts partial labels accordingly.
For each benchmark dataset, we employ a multi-layer perceptron (MLP) of five layers with $784-300-300-300-300-10$ units and use ReLU as an activation function.
This model is optimized using stochastic gradient descent (SGD) algorithm \citep{robbins1951sgd} with a learning rate of $0.1$, a momentum of $0.9$, and a weight decay at $0.001$, $0.0001$, and $0.00001$ for MNIST, Kuzushiji-MNIST, and Fashion-MNIST, respectively. The model is trained for 100 epochs and the learning rate is adjusted using cosine annealing \citep{loshchilov2017sgdr}. 
For real-world datasets, a softmax regression model is used. This model is optimized with the Adam  optimizer \citep{kingma2017adam} with a learning rate of 0.1, 0.1, 0.01, 0.01, and 0.01 and weight decay at $10^{-10}$, $10^{-6}$, $10^{-10}$, $10^{-2}$, and $10^{-6}$ for the Lost, MSCRCv2, BirdSong, Soccer Player, and Yahoo!News dataset, respectively. The model is trained for 200 epochs, and the learning rate is also adjusted using cosine annealing.

\subsection{Results}
\begin{table}[t]  
\caption{Performance comparison of different calibration approaches on benchmark datasets with random contamination ($p=0.1$).}
	\label{tab:result:benchmarks:random:p=0.1}
	\begin{center}  
	\begin{tabular}{l  l  c   c  c }
        \toprule
        & &  FashionMNIST & KMNIST & MNIST \\
        \midrule
        & Train acc. & $96.64\pm0.37$&$98.83\pm0.02$&$99.50\pm0.02$ \\
        & Test acc. & $88.53\pm0.44$&$90.66\pm0.20$&$98.12\pm0.13$\\
        \midrule
        \multirow{2}{*}{$\mathcal{E}_\text{max}$} 
        & Efficiency  & $10.00\pm0.00$&$9.42\pm0.05$&$9.17\pm0.05$ \\ 
        & Coverage & $1.00\pm0.00$&$1.00\pm0.00$&$1.00\pm0.00$ \\
        \midrule
        \multirow{2}{*}{$\mathcal{E}_\text{all}$} 
        & Efficiency & $8.12\pm0.17$&$8.83\pm0.04$&$8.38\pm0.05$ \\
        & Coverage & $1.00\pm0.00$&$1.00\pm0.00$&$1.00\pm0.00$ \\
        \midrule
        \multirow{2}{*}{$\mathcal{E}_\text{mean}$}
        & Efficiency & $1.07\pm0.01$&$1.03\pm0.00$&$1.01\pm0.00$ \\
        & Coverage & $0.91\pm0.01$&$0.92\pm0.00$&$0.98\pm0.00$ \\
        \midrule
        \multirow{2}{*}{$\mathcal{E}_\text{min}$}
        & Efficiency & $0.98\pm0.00$&$0.80\pm0.00$&$0.90\pm0.00$ \\
        & Coverage & $0.88\pm0.00$&$0.79\pm0.00$&$0.90\pm0.00$ \\
        \midrule
        \multirow{2}{*}{$\mathcal{E}_{\mu=0.3}$}
        & Efficiency & $1.08\pm0.01$
& $1.04\pm0.00$
& $1.01\pm0.00$
 \\
        & Coverage &$0.92\pm0.00$
 & $0.92\pm0.00$
 &$0.99\pm0.00$
 \\
        \midrule
        \multirow{2}{*}{$\mathcal{E}_{\mu=0.5}$}
        & Efficiency &$1.04\pm0.00$
 & $0.98\pm0.00$
&$0.99\pm0.00$
  \\
        & Coverage & $0.90\pm0.00$
&$0.90\pm0.00$
&$0.98\pm0.00$
 \\
        \midrule
        \multirow{2}{*}{$\mathcal{E}_{\mu=0.7}$}
        & Efficiency &$1.02\pm0.00$ & $0.93\pm0.00$&$0.98\pm0.00$  \\
        & Coverage &$0.89\pm0.00$ & $0.88\pm0.00$& $0.97\pm0.00$ \\
        \bottomrule
	\end{tabular}
	\end{center}
\end{table}

\begin{table}[t]  
\caption{Performance comparison of different calibration approaches on benchmark datasets with random contamination ($p=0.7$).}
    \label{tab:result:benchmarks:random:p=0.7}
	\begin{center}  
	\begin{tabular}{l  l  c   c  c   }
        \toprule
        & &  FashionMNIST & KMNIST & MNIST \\
        \midrule
        & Train acc. & $88.22\pm0.82$&$93.55\pm0.2$&$97.29\pm0.03$ \\
        & Test acc. &  $85.7\pm0.76$&$82.25\pm0.38$&$96.88\pm0.22$\\
        \midrule
        \multirow{2}{*}{$\mathcal{E}_\text{max}$} 
        & Efficiency  &$9.58\pm0.06$&$9.82\pm0.02$&$9.63\pm0.03$ \\ 
        & Coverage & $1.00\pm0.00$&$1.00\pm0.00$&$1.00\pm0.00$\\
        \midrule
        \multirow{2}{*}{$\mathcal{E}_\text{all}$} 
        & Efficiency &$8.95\pm0.03$&$9.40\pm0.04$&$8.93\pm0.03$ \\
        & Coverage &$1.00\pm0.00$&$1.00\pm0.00$&$1.00\pm0.00$\\
        \midrule
        \multirow{2}{*}{$\mathcal{E}_\text{mean}$}
        & Efficiency & $1.31\pm0.03$&$1.35\pm0.00$&$1.08\pm0.00$ \\
        & Coverage & $0.94\pm0.00$&$0.90\pm0.00$&$0.99\pm0.00$ \\
        \midrule
        \multirow{2}{*}{$\mathcal{E}_\text{min}$}
        & Efficiency & $0.93\pm0.01$&$0.81\pm0.01$&$0.91\pm0.00$\\
        & Coverage & $0.83\pm0.01$&$0.74\pm0.01$&$0.90\pm0.00$\\
        \midrule
        \multirow{2}{*}{$\mathcal{E}_{\mu=0.3}$}
        & Efficiency & $1.20\pm0.02$
& $1.15\pm0.01$
&$1.02\pm0.00$
  \\
        & Coverage &$0.92\pm0.00$
 & $0.86\pm0.00$
 &$0.98\pm0.00$
 \\
        \midrule
        \multirow{2}{*}{$\mathcal{E}_{\mu=0.5}$}
        & Efficiency &$1.09\pm0.01$
 & $1.01\pm0.01$
&$0.99\pm0.00$
  \\
        & Coverage & $0.89\pm0.01$
&$0.83\pm0.00$
&$0.97\pm0.00$
 \\
        \midrule
        \multirow{2}{*}{$\mathcal{E}_{\mu=0.7}$}
        & Efficiency & $1.01\pm0.01$& $0.92\pm0.01$&$0.97\pm0.00$ \\
        & Coverage &$0.87\pm0.01$ &$0.80\pm0.00$ & $0.95\pm0.00$ \\
        \bottomrule
	\end{tabular}
	\end{center}
\end{table}
For the real-world datasets, since there are no separate test datasets, we apply the $80\%-20\%$ train-test split. For synthetic datasets, $10\%$ of the available training data is selected at random for calibration, while for the real-world datasets, a larger subset of $20\%$ is selected.
We fix the miscoverage rate at $\epsilon = 0.1$. The experiments are repeated five times using different random seeds, and the means and standard deviations of the results are reported.

Table \ref{tab:result:benchmarks:random:p=0.1} and Table \ref{tab:result:benchmarks:random:p=0.7} present the results for benchmark datasets, where random contamination is applied with $p$ being set to $0.1$ and $0.7$, respectively. It can be seen that $\mathcal{E}_\text{max}$ and $\mathcal{E}_\text{all}$ results in highly inefficient prediction sets. In fact, when the accuracy of the underlying classifier is high, and the candidate sets are generated in a completely random fashion, then the nonconformity scores of the non-ground-truth labels are so high, resulting in large critical scores and, accordingly, such conservative prediction sets.

Table \ref{tab:result:benchmarks:instance_based} provides the results for the benchmark datasets with instance-dependent contamination. Compared to the random contamination case, the results of $\mathcal{E}_\text{max}$ and $\mathcal{E}_\text{all}$ are less inefficient, while they are the only cases that satisfy the coverage guarantee for all datasets. Finally, the results for the real-world datasets are provided in Table \ref{tab:result:real-world}. Once again, $\mathcal{E}_\text{max}$ and $\mathcal{E}_\text{all}$ result in inefficient large prediction sets. While $\mathcal{E}_\text{mean}$ provides the best results for the synthetic data with random contamination case, the coverage property is not always satisfied for this approach in the two other cases.        

\begin{table}[t]  
\caption{Performance comparison of different calibration approaches on benchmark datasets with instance-dependent contamination.}
	\label{tab:result:benchmarks:instance_based}
	\begin{center}  
	\begin{tabular}{l l c c   c  c  c }
        \toprule
        & & FashionMNIST & KMNIST & MNIST  \\
        \midrule
        & Train acc. & $83.82\pm0.63$&$93.86\pm0.16$&$97.23\pm0.11$  \\
        & Test acc. & $82.48\pm0.69$&$83.9\pm0.31$&$96.93\pm0.17$\\
        \midrule
        \multirow{2}{*}{$\mathcal{E}_\text{max}$}
        & Efficiency & $6.42\pm0.76$&$6.54\pm0.12$&$5.97\pm0.14$ \\
        & Coverage & $1.00\pm0.00$&$1.00\pm0.00$&$1.00\pm0.00$ \\
        \midrule
        \multirow{2}{*}{$\mathcal{E}_\text{all}$} 
        & Efficiency & $3.46\pm0.54$&$5.14\pm0.06$&$4.42\pm0.07$ \\
        & Coverage & $0.99\pm0.00$&$1.00\pm0.00$&$1.00\pm0.00$ \\
        \midrule
        \multirow{2}{*}{$\mathcal{E}_\text{mean}$}
        & Efficiency & $1.14\pm0.03$&$1.11\pm0.00$&$1.02\pm0.00$ \\
        & Coverage & $0.87\pm0.01$&$0.87\pm0.00$&$0.98\pm0.00$\\
        \midrule
        \multirow{2}{*}{$\mathcal{E}_\text{min}$}
        & Efficiency & $0.91\pm0.01$&$0.84\pm0.01$&$0.90\pm0.00$ \\
        & Coverage & $0.77\pm0.01$&$0.76\pm0.01$&$0.89\pm0.00$ \\
        \midrule
        \multirow{2}{*}{$\mathcal{E}_{\mu=0.3}$}
        & Efficiency & $1.23\pm0.04$& $1.09\pm0.00$
& $1.02\pm0.00$
 \\
        & Coverage & $0.89\pm0.01$
&$0.87\pm0.00$
 & $0.98\pm0.00$
 \\
        \midrule
        \multirow{2}{*}{$\mathcal{E}_{\mu=0.5}$}
        & Efficiency & $1.06\pm0.03$& $0.99\pm0.00$
&$1.00\pm0.00$
 \\
        & Coverage & $0.85\pm0.01$
&$0.84\pm0.00$
 & $0.97\pm0.00$
 \\
        \midrule
        \multirow{2}{*}{$\mathcal{E}_{\mu=0.7}$}
        & Efficiency & $0.98\pm0.00$&$0.93\pm0.00$ &$0.97\pm0.00$  \\
        & Coverage &$0.82\pm0.01$ & $0.81\pm0.00$&$0.96\pm0.00$  \\
        \bottomrule
	\end{tabular}
	\end{center}
\end{table}

\begin{table}[t]  
\caption{Performance comparison of different calibration approaches on real-world datasets.}
	\label{tab:result:real-world}
	\begin{center}  
        \resizebox{\columnwidth}{!}{
	\begin{tabular}{l l  c c  c  c  c}
        \toprule
        & & BirdSong & Lost & MSRCv2 & Soccer Player & Yahoo!News \\
        \midrule
        & Train acc. & $74.45\pm0.69$&$86.54\pm3.29$&$54.61\pm0.73$&$52.61\pm0.48$&$69.86\pm0.84$ \\
        & Test acc. & $72.04\pm0.89$&$73.96\pm2.43$&$48.58\pm0.72$&$50.48\pm0.39$&$61.13\pm0.77$ \\
        \midrule
        \multirow{2}{*}{$\mathcal{E}_\text{max}$}
        & Efficiency & $13.00\pm0.00$&$13.27\pm1.05$&$21.4\pm1.51$&$150.31\pm1.09$&$122.84\pm4.11$ \\
        & Coverage &  $1.00\pm0.00$&$1.00\pm0.00$&$0.99\pm0.01$&$0.99\pm0.00$&$0.99\pm0.00$\\
        \midrule
        \multirow{2}{*}{$\mathcal{E}_\text{all}$} 
        & Efficiency &  $9.80\pm1.80$&$10.7\pm0.83$&$18.16\pm0.69$&$137.81\pm1.30$&$48.07\pm1.87$ \\
        & Coverage &  $0.99\pm0.00$&$0.99\pm0.01$&$0.97\pm0.01$&$0.98\pm0.00$&$0.99\pm0.00$ \\
        \midrule
        \multirow{2}{*}{$\mathcal{E}_\text{mean}$}
        & Efficiency &  $2.06\pm0.14$&$1.92\pm0.22$&$2.85\pm0.41$&$20.71\pm4.64$&$2.91\pm0.08$ \\
        & Coverage &  $0.89\pm0.01$&$0.88\pm0.02$&$0.66\pm0.02$&$0.75\pm0.02$&$0.87\pm0.01$ \\
        \midrule
        \multirow{2}{*}{$\mathcal{E}_\text{min}$}
        & Efficiency &  $1.62\pm0.15$&$1.56\pm0.15$&$2.15\pm0.24$&$16.12\pm3.61$&$2.48\pm0.10$ \\
        & Coverage &  $0.84\pm0.02$&$0.86\pm0.02$&$0.63\pm0.03$&$0.73\pm0.02$&$0.84\pm0.01$ \\
        \midrule
        \multirow{2}{*}{$\mathcal{E}_{\mu=0.3}$}
        & Efficiency & $2.23\pm0.12$ & $2.15\pm0.26$ & $3.07\pm0.28$
 & $23.7\pm4.35$ & $3.17\pm0.11$
\\
        & Coverage & $0.90\pm0.01$ & $0.89\pm0.02$ & $0.67\pm0.02$ & $0.76\pm0.02$ & $0.88\pm0.01$
\\
        \midrule
        \multirow{2}{*}{$\mathcal{E}_{\mu=0.5}$}
        & Efficiency & $1.97\pm0.13$ & $1.90\pm0.20$ & $2.58\pm0.29$ & $20.17\pm4.14$ & $2.87\pm0.08$
\\
        & Coverage & $0.88\pm0.01$ & $0.88\pm0.02$ & $0.65\pm0.02$ & $0.75\pm0.02$
 & $0.87\pm0.01$
\\
        \midrule
        \multirow{2}{*}{$\mathcal{E}_{\mu=0.7}$}
        & Efficiency & $1.81\pm0.15$ & $1.74\pm0.16$ & $2.40\pm0.21$ & $18.53\pm3.72$
 & $2.67\pm0.09$
\\
        & Coverage & $0.86\pm0.01$ & $0.87\pm0.02$ & $0.65\pm0.02$ & $0.74\pm0.02$
 & $0.85\pm0.01$
\\
        \bottomrule
	\end{tabular}
        }
	\end{center}
\end{table}

\section{Conclusion}
This paper bridges two popular machine learning frameworks, namely conformal prediction and partial label learning. We propose an extension to conformal prediction, which allows it to handle training and calibration data that are only partially labeled. This is an essential extension as such data arises in many real-world applications, such as web mining, image annotation, text classification, etc., where obtaining complete label information may be difficult or expensive, and equipping predictions with a notion of uncertainty is of utmost importance. 

Since this is the first paper dealing with partial label data for conformal prediction, there are various open problems and research directions to pursue in future work. For example, it is worth exploring whether there is room for enhancing the computation of the nonconformity scores in the calibration step. Indeed, while we theoretically show that the prediction sets constructed by the proposed approaches inherit the validity of the conformal prediction under certain assumptions, their efficiency still needs to be improved. Moreover, as our theoretical validity results rely on the specific properties of nonconformity scores derived from probabilistic classifiers, another interesting contribution would be a generalization of these results to other types of nonconformity measures.

\clearpage
\section*{Acknowledgment}
Alireza Javanmardi was supported by the Deutsche Forschungsgemeinschaft (DFG, German Research Foundation): Project number 451737409. Yusuf Sale was supported by the DAAD programme Konrad Zuse Schools of Excellence in Artificial Intelligence, sponsored by the Federal Ministry of Education and Research.
\bibliographystyle{unsrtnat}
\bibliography{references}  

\newpage
\appendix
\renewcommand\thesection{Appendix \Alph{section}.}
\section{}\label{sec:appendix:supermodel}
 An MLP with $784-100-10$ units with ReLU activation functions is used as a supermodel for the MNIST, Kuzushiji-MINST, and Fashion-MNIST datasets. Table \ref{tab:appendix:supermodels} reports the train and test accuracies of these supermodels.

\begin{table}[h!]
\caption{Accuracies of the supermodels used for instance-dependent contamination.}
\label{tab:appendix:supermodels}
\centering
\begin{tabular}{@{}lcccc@{}}
\toprule
               & FashionMNIST & KMNIST & MNIST \\ \midrule
Train accuracy & 85.13        & 93.70  & 96.56    \\
Test accuracy  & 82.88        & 78.35  & 94.77    \\ \bottomrule
\end{tabular}
\end{table}
Note that the only purpose of the supermodels is to convert precise datasets into partially labeled data. Moreover, supermodel training is independent of partial label learning and calibration steps. Indeed, the supermodel for each benchmark dataset is trained using the training set of that data. Subsequently, the resulting supermodel is utilized to generate partial labels for the same set. The resulting contaminated set will, later on, be divided into proper training and calibration subsets. These subsets will be employed in the partial label learning and calibration steps, respectively. 

\section{}\label{sec:appendix:lemma}
\renewcommand\thesection{\Alph{section}}
\begin{lem}\label{lemma}
Consider a set $\mathcal{E}_1$ and its $\lceil (1+|\mathcal{E}_1|)(1-\epsilon)\rceil$ smallest value, $q(\mathcal{E}_1, \epsilon)$. Suppose we add $t_l\geq2$ elements that are less than or equal to $q(\mathcal{E}_1, \epsilon)$ and $t_r \geq t_l$ elements that are greater than $q(\mathcal{E}_1, \epsilon)$ to form a new set $\mathcal{E}_2$. If $\epsilon \leq \frac{1}{4}$, then it is guaranteed that $q(\mathcal{E}_2, \epsilon) \geq q(\mathcal{E}_1, \epsilon)$.
\end{lem}
\begin{proof}
    Let $d$ be the difference between $t_r$ and $t_l$, i.e., $t_r= t_l + d$. For $q(\mathcal{E}_2, \epsilon) \geq q(\mathcal{E}_1, \epsilon)$ to be true, we need the following to hold:
    \begin{align*}
        \lceil (1+|\mathcal{E}_2|)(1-\epsilon)\rceil &\geq \lceil (1+|\mathcal{E}_1|)(1-\epsilon)\rceil + t_l \Rightarrow\\
         (1+|\mathcal{E}_1| + t_l+t_r)(1-\epsilon) &\geq  (1+|\mathcal{E}_1|)(1-\epsilon) + t_l +1 \Rightarrow\\
         ( t_l+t_r)(1-\epsilon) &\geq t_l +1 \Rightarrow\\
         \epsilon &\leq \frac{t_l +d -1}{2t_l+d} = \frac{1}{2} + \frac{d/2-1}{2t_l+d} \geq \frac{1}{4},
    \end{align*}
    where the last inequality comes from the fact that $d\geq 0$ and $t_l\geq 2$.
\end{proof}

\end{document}